\documentclass[letter]{article}

\usepackage{nips12submit_e,times}
\usepackage{amsmath}
\usepackage{amssymb}
\usepackage{amsthm}
\usepackage{psfrag}
\usepackage{mathrsfs}
\usepackage[latin1]{inputenc}
\usepackage{booktabs}
\usepackage{algorithmic}
\usepackage[square,comma,numbers,sort&compress]{natbib}
\usepackage{graphicx}

\newcommand{\field}[1]{\mathbb{#1}}

\newcommand{\set}[1]{\mathcal{#1}}

\newcommand{\reals}{\field{R}}
\newcommand{\comment}[1]{}

\theoremstyle{plain}
\newtheorem{theorem}{Theorem}

\theoremstyle{definition}

\newcommand{\drawnfrom}{\sim}       
\newcommand{\expect}{\mathbf{E}}    

\title{A Nonparametric Conjugate Prior Distribution for the Maximizing Argument of a Noisy Function}

\author{%
Pedro A. Ortega\\
Max Planck Institute for Intelligent Systems\\
Max Planck Institute for Biolog. Cybernetics\\
\texttt{pedro.ortega@tuebingen.mpg.de}
\AND Jordi Grau-Moya \\
Max Planck Institute for Intelligent Systems\\
Max Planck Institute for Biolog. Cybernetics\\
\texttt{jordi.grau@tuebingen.mpg.de}
\And Tim Genewein \\
Max Planck Institute for Intelligent Systems\\
Max Planck Institute for Biolog. Cybernetics\\
\texttt{tim.genewein@tuebingen.mpg.de}
\And David~Balduzzi \\
Max Planck Institute for Intelligent Systems\\
\texttt{david.balduzzi@tuebingen.mpg.de}
\And Daniel A. Braun \\
Max Planck Institute for Intelligent Systems\\
Max Planck Institute for Biolog. Cybernetics\\
\texttt{daniel.braun@tuebingen.mpg.de}}

\begin{document}
\nipsfinalcopy
\maketitle

\begin{abstract}%
We propose a novel Bayesian approach to solve stochastic optimization problems that involve fnding extrema of noisy, nonlinear functions. Previous work has focused on representing possible functions explicitly, which leads to a two-step procedure of first, doing inference over the function space and second, finding the extrema of these functions. Here we skip the representation step and directly model the distribution over extrema. To this end, we devise a non-parametric conjugate prior based on a kernel regressor. The resulting posterior distribution directly captures the uncertainty over the maximum of the unknown function. We illustrate the effectiveness of our model by optimizing a noisy, high-dimensional, non-convex objective function.
\end{abstract}%


\section{Introduction}

Historically, the fields of statistical inference and stochastic optimization have often developed their own specific methods and approaches. Recently, however, there has been a growing interest in applying inference-based methods to optimization problems and vice versa \cite{Brochu2009, Rawlik2010, Shapiro2000, Kappen2012}. Here we consider stochastic optimization problems where we observe noise-contaminated values from an unknown nonlinear function and we want to find the input that maximizes the expected value of this function.

The problem statement is as follows. Let $\set{X}$ be a metric space. Consider a stochastic function $f: \set{X} \rightsquigarrow \reals$ mapping a test point $x \in \set{X}$ to real values $y \in \reals$ characterized by the conditional pdf $P(y|x)$. Consider the mean function
\begin{equation}
\label{eq:objectivefunction}
    \bar{f}(x) := \expect[y|x] = \int y P(y|x) \,dy.
\end{equation}
The goal consists in modeling the optimal test point
\begin{equation}
\label{eq:argmax}
    x^\ast := \arg \max_x \{ \bar{f}(x) \}.
\end{equation}

Classic approaches to solve this problem are often based on stochastic approximation methods \cite{Kushner1997}. Within the context of statistical inference, Bayesian optimization methods have been developed where a prior distribution over the space of functions is assumed and uncertainty is tracked during the entire optimization process \cite{Mockus1994, Lizotte2008}. In particular, non-parametric Bayesian approaches such as Gaussian Processes have been applied for derivative-free optimization \cite{Jones1998, Osborne2009}, also within the context of the continuum-armed bandit problem \cite{Srinivas2010}. Typically, these Bayesian approaches aim to explicitly represent the unknown objective function of \eqref{eq:objectivefunction} by entertaining a posterior distribution over the space of objective functions. In contrast, we aim to model directly the distribution of the maximum of~\eqref{eq:argmax} conditioned on observations.

The paper is structured as follows. Section~\ref{sec:description} gives a brief description of the model suitable for direct implementation. The model is then derived in Section~\ref{sec:derivation}. Section~\ref{sec:results} presents experimental results. Section~\ref{sec:results} concludes.

\section{Description of the Model}\label{sec:description}

\begin{figure}[tbp]
\begin{center}
    \includegraphics[width=14cm]{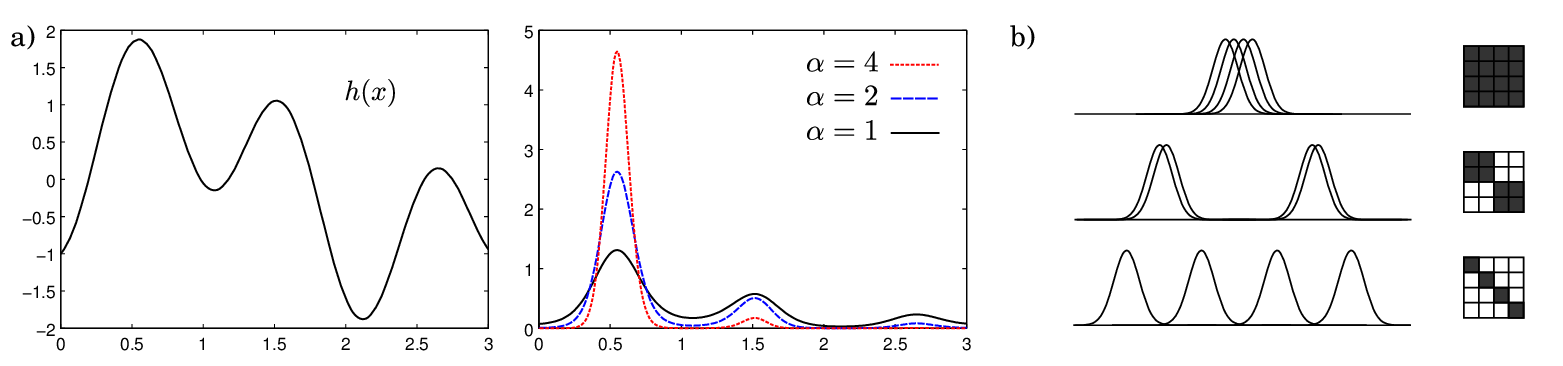}
    \caption{a) Given an estimate $h$ of the mean function $\bar{f}$ (left), a simple probability density function over the location of the maximum $x^\ast$ is obtained using the transformation $P(x^\ast) \propto \exp\{ \alpha h(x^\ast) \}$, where $\alpha > 0$ plays the role of the precision (right). b) Illustration of the Gramian matrix for different test locations. Locations thar are close to each other produce large off-diagonal entries.}\label{fig:intuition}
\end{center}
\end{figure}

Our model is intuitively straightforward and easy to implement\footnote{Implementations can be downloaded from http://www.adaptiveagents.org/argmaxprior}. Let $h(x): \set{X} \rightarrow \reals$ be an estimate of the mean $\bar{f}(x)$ constructed from data $\set{D}_t := \{(x_i, y_i)\}_{i=1}^t$ (Figure~\ref{fig:intuition}a, left). This estimate can easily be converted into a posterior pdf over the location of the maximum by first multiplying it with a precision parameter $\alpha>0$ and then taking the normalized exponential (Figure~\ref{fig:intuition}a, right)
\[
    P(x^\ast|\set{D}_t)
    \propto \exp\{ \alpha \cdot h(x^\ast) \}.
\]
In this transformation, the precision parameter $\alpha$ controls the certainty we have over our estimate of the maximizing argument: $\alpha \approx 0$ expresses almost no certainty, while $\alpha \rightarrow \infty$ expresses certainty. The rationale for the precision is: the more \emph{distinct} inputs we test, the higher the precision---testing the same (or similar) inputs only provides \emph{local} information and therefore should not increase our knowledge about the \emph{global} maximum. A simple and effective way of implementing this idea is given by
\begin{equation}\label{eq:argmaxprior}
    P(x^\ast|\set{D}_t)
    \propto
    \exp\biggl\{
    \rho
    \cdot \underbrace{ 
	\biggl( \xi +
        t \cdot \frac{ \sum_i K(x_i, x_i) }{ \sum_i \sum_j K(x_i, x_j) }
        \biggr)
    }_\text{effective \# of locations}
    \cdot \underbrace{
        \frac{ \sum_i K(x_i, x^\ast) y_i + K_0(x^\ast) y_0(x^\ast)}
             { \sum_i K(x_i, x^\ast) + K_0(x^\ast) }
    }_\text{estimate of $\bar{f}(x^\ast)$}
    \biggr\},
\end{equation}
where $\rho$, $\xi$, $K$, $K_0$ and $y_0$ are parameters of the estimator: $\rho > 0$ is the precision we gain for each new distinct observation; $\xi > 0$ is the number of prior points; $K: \set{X} \times \set{X} \rightarrow \reals^{+}$ is a finite, symmetric kernel function; $K_0: \set{X} \rightarrow \reals^{+}$ is a prior precision function; and $y_0: \set{X} \rightarrow \reals$ is a prior estimate of $\bar{f}$.

In~\eqref{eq:argmaxprior}, the mean function $\bar{f}$ is estimated with a kernel regressor \cite{Hastie2009}, and the total effective number of locations is calculated as the sum of the prior locations $\xi$ and the number of distinct locations in the data $\set{D}_t$. The latter is estimated by multiplying the number of data points $t$ with the coefficient
\[
  \frac{ \sum_i K(x_i, x_i) }{ \sum_i \sum_j K(x_i, x_j) } \in (0,1],
\]
i.e.\ the ratio between the trace of the Gramian matrix $(K(x_i,x_j))_{i,j}$ and the sum of its entries. Inputs that are very close to each other will have overlapping kernels, resulting in large off-diagonal entries of the Gramian matrix---hence decreasing the number of distinct locations (Figure~\ref{fig:intuition}b).

The expression for the posterior can be calculated, up to a constant factor, in quadratic time in the number of observations. It can therefore be easily combined with Markov chain Monte Carlo methods (MCMC) to implement stochastic optimizers as illustrated in Section~\ref{sec:results}.

\section{Derivation of the Model}\label{sec:derivation}

\subsection{Function-Based, Indirect Model}

Our first task is to derive an \emph{indirect} Bayesian model for the optimal test point that builds its estimate via the underlying function space. Let $\set{G}$ be the set of hypotheses, and assume that each hypothesis $g \in \set{G}$ corresponds to a stochastic mapping $g: \set{X} \rightsquigarrow \reals$. Let $P(g)$ be the prior\footnote{For the sake of simplicity, we neglect issues of measurability of $\set{G}$.} over $\set{G}$ and let the likelihood be $P(\{y_t\}|g,\{x_t\}) = \prod_t P(y_t|g, x_t)$. Then, the posterior of $g$ is given by
\begin{equation}\label{eq:post-function}
    P(g|\{y_t\},\{x_t\})
    = \frac{P(g) P(\{y_t\}|g,\{x_t\})}{P(\{y_t\}|\{x_t\})}
    = \frac{P(g) \prod_t P(y_t|g, x_t)}{P(\{y_t\}|\{x_t\})}.
\end{equation}
For each $x^\ast \in \set{X}$, let $\set{G}(x^\ast) \subset \set{G}$ be the subset of functions such that for all $g \in \set{G}(x^\ast)$, $x^\ast = \arg \max_x \{\bar{g}(x)\}$\footnote{Note that we assume that the mean function $\bar{g}$ is bounded and that it has a unique maximizing test point.}. Then, the posterior over the optimal test point $x^\ast$ is given by
\begin{equation}\label{eq:post-marginal}
    P(x^\ast|\{y_t\},\{x_t\})
    = \int_{\set{G}(x^\ast)} P(g|\{y_t\},\{x_t\}) \,dg,
\end{equation}
This model has two important drawbacks: (a) it relies on modeling the entire function space~$\set{G}$, which is potentially much more complex than necessary; (b) it requires calculating the integral~\eqref{eq:post-marginal}, which is intractable for virtually all real-world problems.

\subsection{Domain-Based, Direct Model}

We want to arrive at a Bayesian model that bypasses the integration step suggested by~\eqref{eq:post-marginal} and directly models the location of optimal test point~$x^\ast$. The following theorem explains how this \emph{direct model} relates to the previous model.

\begin{theorem}\label{theo:bayesian-model}
The Bayesian model for the optimal test point $x^\ast$ is given by
\begin{align*}
    P(x^\ast) &=
        \int_{\set{G}(x^\ast)} P(g) \,dg
    &&\text{(prior)} \\
    P(y_t|x^\ast, x_t, \set{D}_{t-1})
    &= \frac{ \int_{\set{G}(x^\ast)}
        P(y_t|g, x_t) P(g) \prod_{k=1}^{t-1} P(y_k|g, x_k) \,dg }
        { \int_{\set{G}(x^\ast)} P(g) \prod_{k=1}^{t-1} P(y_k|g, x_k) \,dg },
    &&\text{(likelihood)}
\end{align*}
where $\set{D}_t := \{(x_k, y_k)\}_{k=1}^t$ is the set of past tests.
\end{theorem}
\begin{proof}
Using Bayes' rule, the posterior distribution $P(x^\ast|\{y_t\},\{x_t\})$ can be rewritten as
\begin{equation}\label{eq:post-opt}
    \frac{P(x^\ast) \prod_t P(y_t|x^\ast,x_t,\set{D}_{t-1})}
         {P(\{y_t\}|\{x_t\})}.
\end{equation}
Since this posterior is equal to~\eqref{eq:post-marginal}, one concludes (using~\eqref{eq:post-function}) that
\[
    P(x^\ast) \prod_t P(y_t|x^\ast,x_t,\set{D}_{t-1})
    = \int_{\set{G}(x^\ast)} P(g) \prod_t P(y_t|g, x_t) \,dg.
\]
Note that this expression corresponds to the joint $P(x^\ast, \{y_t\}|\{x_t\})$. The prior $P(x^\ast)$ is obtained by setting $t=0$. The likelihood is obtained as the fraction
\[
    P(y_t|x^\ast, x_t, \set{D}_{t-1})
    = \frac{ P(x^\ast, \{y_k\}_{k=1}^{t\quad}| \{x_k\}_{k=1}^{t\quad}) }
           { P(x^\ast, \{y_k\}_{k=1}^{t-1}| \{x_k\}_{k=1}^{t-1}) },
\]
where it shall be noted that the denominator $P(x^\ast, \{y_k\}_{k=1}^{t-1}| \{x_k\}_{k=1}^{t-1})$ doesn't change if we add the condition~$x_t$.
\end{proof}

From Theorem~\ref{theo:bayesian-model} it is seen that although the likelihood model $P(y_t|g, x_t)$ for the indirect model is i.i.d.\ at each test point, the likelihood model $P(y_t|x^\ast, x_t, \set{D}_{t-1})$ for the direct model depends on the past tests $\set{D}_{t-1}$, that is, \emph{it is adaptive}. More critically though, the likelihood function's internal structure of the direct model corresponds to an integration over function space as well---thus inheriting all the difficulties of the indirect model.

\subsection{Abstract Properties of the Likelihood Function}

There is a way to bypass modeling the function space explicitly if we make a few additional assumptions. We assume that for any $g \in \set{G}(x^\ast)$, the mean function $\bar{g}$ is continuous and has a unique maximum. Then, the crucial insight consists in realizing that the value of the mean function $\bar{g}$ inside a sufficiently small neighborhood of $x^\ast$ is larger than the value outside of it (see Figure~\ref{fig:assumptions}a).

We assume that, for any $\delta > 0$ and any $z \in \set{X}$, let $B_\delta(z)$ denote the open $\delta$-ball centered on $z$. The functions in $\set{G}$ fulfill the following properties:
\begin{enumerate}
    \item[a.] \emph{Continuous:} Every function $g \in \set{G}$ is such that its mean $\bar{g}$ is continuous and bounded.
    \item[b.]\emph{Maximum:} For any $x^\ast \in \set{X}$, the functions $g \in \set{G}(x^\ast)$ are such that for all $\delta > 0$ and all $z \notin B_\delta(x^\ast)$, $\bar{g}(x^\ast) > \bar{g}(z)$.
\end{enumerate}

\begin{figure}[tbp]
\begin{center}
    \footnotesize
    \psfrag{la}[c]{a)}
    \psfrag{lb}[c]{b)}
    \psfrag{lc}[c]{c)}
    \psfrag{x1}[c]{$x^\ast$}
    \psfrag{x2}[c]{$x^\ast_1$}
    \psfrag{x3}[c]{$x^\ast_2$}
    \psfrag{y1}[c]{$\phantom{-y}$}
    \psfrag{y2}[c]{$\phantom{-}0$}
    \psfrag{y3}[c]{$\phantom{-y}$}
    \includegraphics{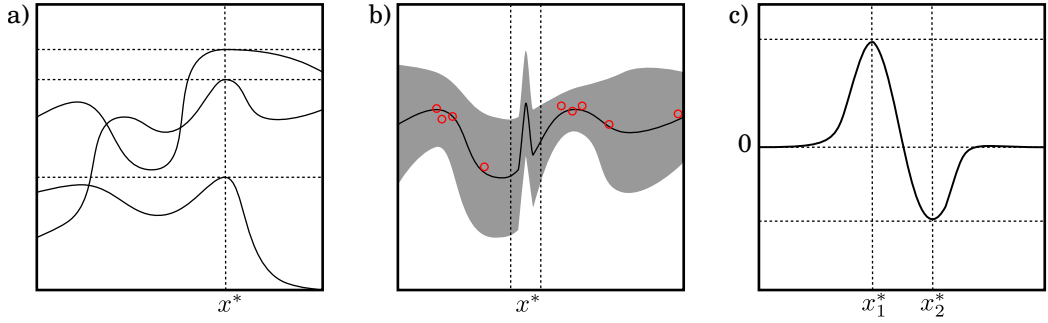}
    \caption{Illustration of assumptions. a) Three functions from $\set{G}(x^\ast)$. They all have their maximum located at $x^\ast \in \set{X}$. b) Schematic representation of the likelihood function of $x^\ast \in \set{X}$ conditioned on a few observations. The curve corresponds to the mean and the shaded area to the confidence bounds. The density inside of the neighborhood is unique to the hypothesis $x^\ast$, while the density outside is shared amongst all the hypotheses. c) The log-likelihood ratio of the hypotheses $x^\ast_1$ and $x^\ast_2$ as a function of the test point $x$. The kernel used in the plot is Gaussian.}\label{fig:assumptions}
\end{center}
\end{figure}

Furthermore, we impose a symmetry condition on the likelihood function. Let $x^\ast_1$ and $x^\ast_2$ be in $\set{X}$, and consider their associated equivalence classes $\set{G}(x^\ast_1)$ and $\set{G}(x^\ast_2)$. There is no reason for them to be very different: in fact, they should virtually be indistinguishable outside of the neighborhoods of~$x^\ast_1$ and~$x^\ast_2$. It is only inside of the neighborhood of $x^\ast_1$ when $\set{G}(x^\ast_1)$ becomes distinguishable from the other equivalence classes because the functions in $\set{G}(x^\ast_1)$ systematically predict higher values than the rest. This assumption is illustrated in Figure~\ref{fig:assumptions}b. In fact, taking the log-likelihood ratio of two competing hypotheses
\[
  \log \frac{ P(y_t|x^\ast_1, x_t, \set{D}_{t-1}) }{ P(y_t|x^\ast_2, x_t, \set{D}_{t-1}) }
\]
for a given test location $x_t$ should give a value equal to zero unless $x_t$ is inside of the vicinity of $x^\ast_1$ or $x^\ast_2$ (see Figure~\ref{fig:assumptions}c). In other words, the amount of evidence a hypothesis gets when the test point is outside of its neighborhood is essentially zero (i.e. it is the same as the amount of evidence that most of the other hypotheses get).

\subsection{Likelihood and Conjugate Prior}

Following our previous discussion, we propose the following likelihood model. Given the previous data $\set{D}_{t-1}$ and a test point $x_t \in \set{X}$, the likelihood of the observation $y_t$ is
\begin{equation}\label{eq:likelihood}
  P(y_t|x^\ast, x_t, \set{D}_{t-1})
  = \frac{1}{Z(x_t, \set{D}_{t-1})}
        \lambda(y_t|x_t, \set{D}_{t-1})
        \exp\bigl\{ \alpha_t \cdot h_t(x^\ast) - \alpha_{t-1} \cdot h_{t-1}(x^\ast) \bigr\},
\end{equation}
where: $Z(x_t, \set{D}_{t-1})$ is a normalizing constant; $\lambda(y_t|x_t, \set{D}_{t-1})$ is a posterior probability over $y_t$ given $x_t$ and the data $\set{D}_{t-1}$; $\alpha_t$ is a precision measuring the knowledge we have about the whole function given by
\[
  \alpha_0 := \rho \cdot \xi 
  \qquad \text{and} \qquad
  \alpha_t := \rho \cdot \Bigl( \xi + \frac{ \sum_i K(x_i, x_i) }{ \sum_i \sum_j K(x_i, x_j) } \Bigr)
\]
where $\rho>0$ is a precision scaling parameter; $\xi > 0$ is a parameter representing the number prior locations tested; and $h_t$ is an estimate of the mean function $\bar{f}$ given by
\[
  h_0(x^\ast) := y_0(x^\ast)
  \qquad \text{and} \qquad
  h_t(x^\ast) := \frac{ \sum_{i=1}^t K(x_i, x^\ast) y_i + K_0(x^\ast) y_0(x^\ast) }
    { \sum_{i=1}^t K(x_i, x^\ast) + K_0(x^\ast)}.
\]
In the last expression, $y_0$ corresponds to a prior estimate of $\bar{f}$ with prior precision $K_0$. Inspecting~\eqref{eq:likelihood}, we see that the likelihood model favours positive changes to the estimated mean function \emph{from} new, unseen test locations. The pdf $\lambda(y_t|x_t, \set{D}_{t-1})$ does not need to be explicitly defined, as it will later drop out when computing the posterior. The only formal requirement is that it should be independent of the hypothesis $x^\ast$. 

We propose the conjugate prior
\begin{equation}\label{eq:prior}
  P(x^\ast) 
  = \frac{1}{Z_0} \exp\{ \alpha_0 \cdot g_0(x^\ast) \}
  = \frac{1}{Z_0} \exp\{ \xi \cdot y_0(x^\ast) \}.  
\end{equation}
The conjugate prior just encodes a prior estimate of the mean function. In a practical optimization application, it serves the purpose of guiding the exploration of the domain, as locations $x^\ast$ with high prior value $y_0(x^\ast)$ are more likely to contain the maximizing argument. 

Given a set of data points $\set{D}_t$, the prior~\eqref{eq:prior} and the likelihood~\eqref{eq:likelihood} lead to a posterior given by
\begin{align}
  \nonumber
  P(x^\ast|\set{D}_t)
  &= \frac{ P(x^\ast) \prod_{k=1}^t P(y_k|x^\ast, x_k, \set{D}_{k-1}) }
         { \int_\set{X} P(x') \prod_{k=1}^t P(y_k|x', x_k, \set{D}_{k-1}) \, dx' }
  \\ \nonumber &= 
    \frac{
      \exp\bigl\{ \sum_{k=1}^t \alpha_k \cdot h_k(x^\ast) - \alpha_{k-1} \cdot h_{k-1}(x^\ast) \bigr\}
      Z_0^{-1} \prod_{k=1}^t Z(x_k, \set{D}_{k-1})^{-1}
    }{
      \int_\set{X}
      \exp\bigl\{ \sum_{k=1}^t \alpha_k \cdot h_k(x') - \alpha_{k-1} \cdot h_{k-1}(x') \bigr\}
      Z_0^{-1} \prod_{k=1}^t Z(x_k, \set{D}_{k-1})^{-1}
      \, dx'
    }
   \\ \label{eq:posterior} &=
    \frac{ 
      \exp\bigl\{ \alpha_t \cdot h_t(x^\ast) \bigr\} 
    }{
      \int_\set{X} \exp\bigl\{ \alpha_t \cdot h_t(x') \bigr\} \, dx'
    }.
\end{align}
Thus, the particular choice of the likelihood function guarantees an analytically compact posterior expression. In general, the normalizing constant in~\eqref{eq:posterior} is intractable, which is why the expression is only practical for relative comparisons of test locations. Substituting the precision~$\alpha_t$ and the mean function estimate~$h_t$ yields
\[
  P(x^\ast|\set{D}_t)
    \propto
    \exp\biggl\{
    \rho
    \cdot
	\biggl( \xi +
        t \cdot \frac{ \sum_i K(x_i, x_i) }{ \sum_i \sum_j K(x_i, x_j) }
        \biggr)
    \cdot
        \frac{ \sum_i K(x_i, x^\ast) y_i + K_0(x^\ast) y_0(x^\ast)}
             { \sum_i K(x_i, x^\ast) + K_0(x^\ast) }
    \biggr\}.
\]

\section{Experimental Results}\label{sec:results}

\subsection{Parameters.}
We have investigated the influence of the parameters on the resulting posterior probability distribution. Figure~\ref{fig:parameters} shows how the choice of the precision~$\rho$ and the kernel width~$\sigma$ affect the shape of the posterior probability density. We have used the Gaussian kernel
\begin{equation}\label{eq:kernel}
    K(x, x^\ast) = \exp\Bigl\{-\frac{1}{2\sigma^2} (x-x^\ast)^2\Bigr\}.
\end{equation}
In this figure, 7~data points are shown, which were drawn as $y \drawnfrom N(f(x), 0.3)$, where the mean function is
\begin{equation}\label{eq:objective}
f(x) = \cos(2x+\tfrac{3}{2}\pi) + \sin(6x+\tfrac{3}{2}\pi).
\end{equation}
The functions $K_0$ and $y_0$ were chosen as
\begin{equation}\label{eq:prior-suff-stat}
    K_0(x) = 1 \qquad\text{and}\qquad y_0(x) = -\frac{1}{2\sigma^2_0}(x-\mu_0)^2,
\end{equation}
where the latter corresponds to the logarithm of a Gaussian with mean~$\mu_0=1.5$ and variance~\mbox{$\sigma^2_0 = 5$}. Choosing a higher value for~$\rho$ leads to sharper updates, while higher values for the kernel width~$\sigma$ produce smoother posterior densities.

\begin{figure}[tbp]
\begin{center}
    \footnotesize
    \psfrag{la}[c]{a)}
    \psfrag{lb}[c]{b)}
    \psfrag{lc}[c]{c)}
    \psfrag{x1}[c]{$\set{X}$}
    \psfrag{x2}[c]{$\set{X}$}
    \includegraphics[width=14cm]{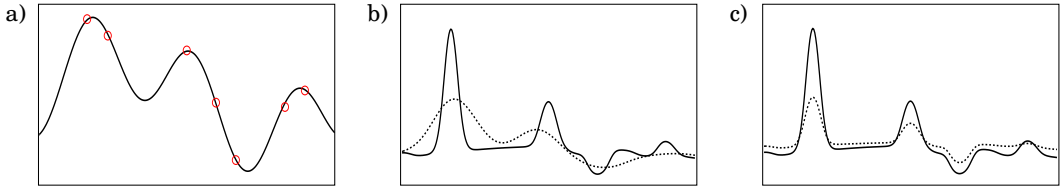}
    \caption{Effect of the change of parameters on the posterior density over the location of the maximizing test point. 
Panel~(a) shows the 7 data points drawn from the noisy function (solid curve). Panel~(b) shows the effect of diminishing the precision on the posterior, where solid and
shaded curves correspond to $\rho=0.2$ and $\rho=0.1$ respectively. Panel~(c) shows the effect of increasing the width of the kernel (here, Gaussian).
The solid and dotted curves correspond to $\sigma = 0.01$ and $\sigma=0.1$ respectively.}\label{fig:parameters}
\end{center}
\end{figure}

\subsection{Application to Optimization.}

\paragraph{Comparison to Gaussian Process UCB.}

We have used the model to optimize the same function~\eqref{eq:objective} as in our preliminary tests but with higher additive noise equal to one. This is done by sampling the next test point $x_t$ directly from the posterior density over the optimum location $P(x^\ast|\set{D}_t)$, and then using the resulting pair $(x_t,y_t)$ to recursively update the model. Essentially, this procedure corresponds to Bayesian control rule/Thompson~sampling \cite{May2011, OrtegaBraun2010}.

We compared our method against a Gaussian Process optimization method using an upper confidence bound (UCB) criterion \cite{Srinivas2010}. The parameters for the GP-UCB were set to the following values: observation noise $\sigma_n=0.3$ and length scale $\ell=0.3$. For the constant that trades off exploration and exploitation we followed Theorem $1$ in \cite{Srinivas2010} which states  $\beta_t=2\log(|D|t^2\pi^2/6\delta)$ with $\delta=0.5$. We have implemented our proposed method with a Gaussian kernel as in~\eqref{eq:kernel} with width~$\sigma^2=0.05$. The prior sufficient statistics are exactly as in~\eqref{eq:prior-suff-stat}. The precision parameter was set to $\rho = 0.3$. 

Simulation results over ten independent runs are summarized in Figure~\ref{fig:bandit}. We show the time-averaged observation values $y$ of the noisy function evaluated at test locations sampled from the posterior. Qualitatively, both methods show very similar convergence (on average), however our method converges faster and with a slightly higher variance. 

\begin{figure}[tbp]
\begin{center}
    \footnotesize
    \psfrag{la}[c]{a)}
    \psfrag{lb}[c]{b)}
    \psfrag{lc}[c]{c)}
    \psfrag{l1}[c]{$\rho=0.1$}
    \psfrag{l2}[c]{$\rho=0.5$}
    \includegraphics[width=10cm]{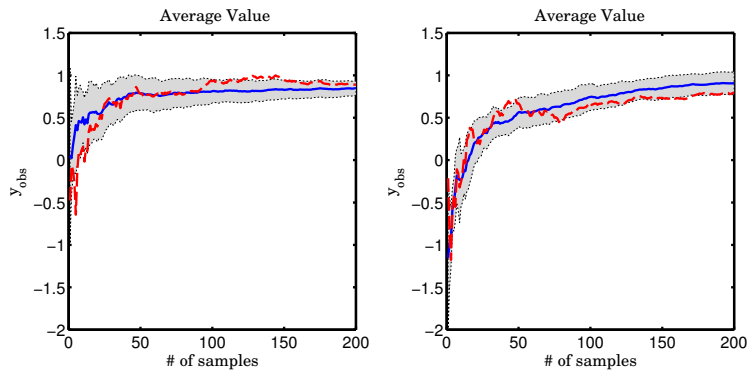}
    \caption{Observation values obtained by sampling from the posterior over the maximizing argument (left panel) and according to GP-UCB (right panel). The solid blue curve corresponds to the time-averaged function value, averaged over ten runs. The gray area corresponds to the error bounds (1~standard deviation), and the dashed curve in red shows the time-average of a single run.}\label{fig:bandit}
\end{center}
\end{figure}

\paragraph{High-Dimensional Problem.}

To test our proposed method on a challenging problem, we have designed a non-convex, high-dimensional noisy function with multiple local optima. This \emph{Noisy Ripples} function is defined as 
\[
  f(x) = -\tfrac{1}{1000} \|x-\mu\|^2 + \cos(\tfrac{2}{3}\pi \|x-\mu\|)
\]
where $\mu \in \set{X}$ is the location of the global maximum, and where observations have additive Gaussian noise with zero mean and variance $0.1$. The advantage of this function is that it generalizes well to any number of dimensions of the domain. Figure~\ref{fig:ripples}a illustrates the function for the 2-dimensional input domain. This function is difficult to optimize because it requires averaging the noisy observations and smoothing the ridged landscape in order to detect the underlying quadratic form. 

We optimized the 50-dimensional version of this function using a Metropolis-Hastings scheme to sample the next test locations from the posterior over the maximizing argument. The Markov chain was started at $[20,20,\cdots,20]^T$, executing 120~isotropic Gaussian steps of variance $0.07$ before the point was used as an actual test location. For the arg-max prior, we used a Gaussian kernel with lengthscale $l=2$, precision factor $\rho=1.5$, prior precision $K_0(x^\ast) = 1$ and prior mean estimate~$y_0(x^\ast) = -\tfrac{2}{1000} \|x + 5\|^2$. The goal $\mu$ was located at the origin.

The result of one run is presented in Figure~\ref{fig:ripples}b. It can be seen that the optimizer manages to quickly ($\approx 100$ samples) reach near-optimal performance, overcoming the difficulties associated with the high-dimensionality of the input space and the numerous local optima. Crucial for this success was the choice of a kernel that is wide enough to accurately estimate the mean function. The authors are not aware of any method capable of solving a problem is similar characteristics.

\begin{figure}[tbp]
\begin{center}
    \footnotesize
    \psfrag{la}[c]{a)}
    \psfrag{lb}[c]{b)}
    \psfrag{lc}[c]{c)}
    \psfrag{x1}[c]{$\set{X}$}
    \psfrag{x2}[c]{$\set{X}$}
    \includegraphics[width=14cm]{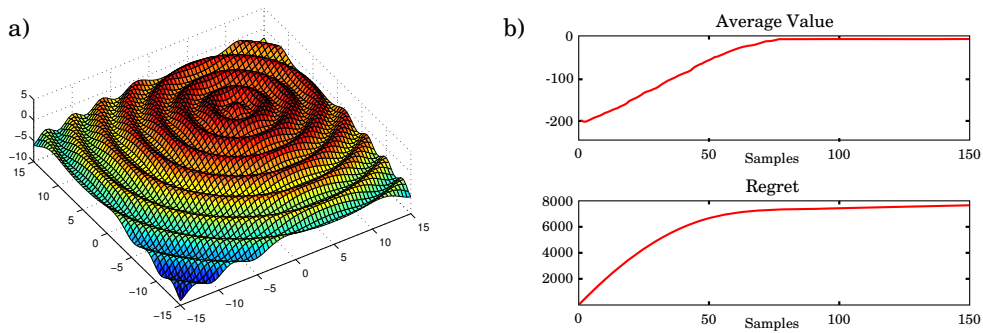}
    \caption{a) The \emph{Noisy Ripples} objective function in 2~dimensions. b) The time-averaged value  and the regret obtained by the optimization algorithm on a 50-dimensional version of the \emph{Noisy Ripples} function.}\label{fig:ripples}
\end{center}
\end{figure}

\section{Discussion \& Conclusions}\label{sec:discussion}

We have proposed a novel Bayesian approach to model the location of the maximizing test point of a noisy, nonlinear function. This has been achieved by directly constructing a probabilistic model over the input space, thereby bypassing having to model the underlying function space---a much harder problem. In particular, we derived a likelihood function that belongs to the exponential family by assuming a form of symmetry in function space. This in turn, enabled us to state a conjugate prior distribution over the optimal test point.

Our proposed model is computationally very efficient when compared to Gaussian process-based (cubic) or UCB-based models (expensive computation of $\arg\max$). The evaluation time of the posterior density scales quadratically in the size of the data. This is due to the calculation of the effective number of previously seen test locations---the kernel regressor requires linear compuation time. However, during MCMC steps, the effective number of test locations does not need to be updated as long as no new observations arrive. 

In practice, one of the main difficulties associated with our proposed method is the choice of the parameters. As in any kernel-based estimation method, choosing the appropriate kernel bandwidth can significantly change the estimate and affect the performance of optimizers that rely on the model. There is no clear rule on how to choose a good bandwidth.   

\comment{
Hamiltonian MCMC
sampling is very important

can solve multidimensional

very faster

scales very well with number of dimensions

no global search

sensitive in choice of parameters

not obvious how to choose parameters

small mh steps

kernel allows to smoother (takes averages) gain information about the whole region
}

In a future research, it will be interesting to investigate the theoretical properties of the proposed nonparametric model, such as the convergence speed of the estimator and its relation to the extensive literature on active learning and bandits.

\newpage
\renewcommand{\bibsection}{\subsubsection*{References}}

\vskip 0.2in
\bibliographystyle{unsrtnat}
\bibliography{bibliography}

\end{document}